\documentclass{bmvc2k}

\usepackage{booktabs}
\usepackage{graphicx}
\usepackage{amsmath}
\usepackage{amssymb}
\usepackage{amsthm}
\usepackage{algorithm}
\usepackage{algorithmic}

\newtheorem{Lemma}{Lemma}

\newcommand{\LL}{\mathbf{L}}

\title{Determinantal Point Process as an alternative to NMS}

\addauthor{Samik Some}{samiks@iitk.ac.in}{1}
\addauthor{Mithun Das Gupta}{migupta@microsoft.com}{2}
\addauthor{Vinay P. Namboodiri}{vpn22@bath.ac.uk}{3}

\addinstitution{
 Dept. of CSE\\
 Indian Institute of Technology, Kanpur\\
 India
}
\addinstitution{
 Microsoft\\
 Hyderabad,\\
 India
}
\addinstitution{
 Dept. of CS\\
 University of Bath\\
 United Kingdom
}

\runninghead{Some, Gupta, Namboodiri}{DPP as NMS Alternative}

\begin{document}

\maketitle

\begin{abstract}
    We present a determinantal point process (DPP) inspired alternative to non-maximum suppression (NMS) which has become an integral step in all state-of-the-art object detection frameworks. DPPs have been shown to encourage diversity in subset selection problems~\cite{GongNIPS14}. We pose NMS as a subset selection problem and posit that directly incorporating DPP like framework can improve the overall performance of the object detection system. We propose an optimization problem which takes the same inputs as NMS, but introduces a novel sub-modularity based diverse subset selection functional. Our results strongly indicate that the modifications proposed in this paper can provide consistent improvements to state-of-the-art object detection pipelines.
\end{abstract}

\section{Introduction}
\label{sec:intro}
Object detection has gained a lot of momentum over the past few years, especially due to its application in a wide variety of fields such as autonomous driving, manufacturing industry, traffic and law enforcement~\cite{HaroonShah2018} applications. The primary approaches for object detection can be loosely divided into a few dominant approaches, including sliding-window Deformable Parts Models~\cite{Felzenszwalb2010,Zhu2010LatentHS}, region proposal with classification~\cite{GirshickDDM13,UijlingsIJCV2013}, and location regression with deep learning~\cite{Overfeat2014,Szegedy13deepneural}. Almost all of the current day object detection frameworks follow a three step process, namely:
(1) proposing a search space of windows, which has mostly converged to the output of a region proposal network (RPN), (2) scoring/
refining the window with a classifier/regressor, and (3)
merging or discarding windows that might belong to the same object.
This last stage is commonly referred to as ``non-maximum suppression'' (NMS)~\cite{GirshickDDM13,KaimingHe16,FRCNN_NIPS2015,Felzenszwalb2010,redmon2015unified,liu2015single}.

NMS is a fairly simple test time post-processing routine. Maintaining parity with some of the published research in this area, we denote the basic NMS step as GreedyNMS~\cite{Felzenszwalb2010,Rothe2014NonmaximumSF,Hosang2017cvpr} in this paper. The GreedyNMS algorithm, greedily selects high scoring detected windows and iteratively discards spatially close-by less confident neighbours with the assumption that the neighbors are likely to cover the same object. Specifically, all the candidate windows are either selected or rejected based on the following procedure: first, the highest-scored window is marked as retained, and all those overlapping with it by more than some threshold (e.g.~30\%) intersection-over-union (IoU) are marked as suppressed; then, the next highest-scored window neither retained nor suppressed is marked as retained, and again all others sufficiently-overlapping candidate windows are marked for rejection. This process is repeated until all windows are marked as either retained or suppressed. The retained windows then constitute the final set of detected proposals.
Although GreedyNMS continues to be the method of choice due to its simplicity, it inherently suffers from significant conceptual shortcomings. GreedyNMS is based on the simple intuition that similar detection windows which are close in spatial sense, should be suppressed. It controls the influence span by a single threshold parameter which is chosen to keep the region of suppression not too wide, since a wide suppression would remove close-by high scoring detected windows that are likely to be false positives that hurt precision. If objects are indeed close to each other, such as persons in crowded scenes, then the windows detected close to each other should be counted as true positives, in which case suppression should be narrow to improve recall. Achieving both these targets with a single tuning parameter seems hard and indeed this inherent limitation is the biggest shortcoming of the GreedyNMS routine.

One of the seminal works in general object detection was the R-CNN model by ~\citet{GirshickDDM13}, which replaced the feature extraction and classifier pipeline by a neural network, resulting in almost two times performance gain on Pascal VOC. Another significant improvement was the F-RCNN model by~\citet{FRCNN_NIPS2015}, which absorbed the object proposal generation into the network, while YOLO~\cite{redmon2015unified} avoided proposals altogether, leading to both speed and quality improvements. A general trend towards end-to-end trainable object detection models has been the norm in recent times. NMS is one step in the object detection pipeline that is based on post-processing. Though a few works have tried to incorporate end-to-end trainable pipelines~\cite{Hosang2017cvpr,WanCVPR2015}, so far it is not widely accepted. We would like to retain the post-processing nature of NMS in order for our approach to be incorporated in any pipeline.

In this work, we propose a principled improvement of the core NMS step by incorporating a DPP cost function in it. This development leads to an overall improvement of the NMS step and can be incorporated to existing NMS implementation with minimal changes. The theoretical guarantees afforded by a DPP based cost function lets us bridge the aforementioned gaps in fundamental ways, namely:
\begin{itemize}
    \item We improve the performance of NMS staying in the standard flow, wherein NMS still stays outside the main neural loop in state-of-the-art (SOTA) object detection implementations,
    \item The proposed system does not need any additional training as in~\cite{Hosang2017cvpr,Azadi_2017_CVPR} or modification of standard cost functions as in~\cite{WanCVPR2015}.
    \item the proposed system works with the same inputs as NMS, namely proposal windows and their score, and introduces a new way to select diverse proposal subsets.
\end{itemize}

\section{Related Work}
Wan et al.~\cite{WanCVPR2015}, proposed to integrate the NMS cost function into the unified loss function of a joint optimization system which had a neural featurizer, a deformable parts model and an NMS block. Since the NMS block was outside the neural loop, this implementation was similar to GreedyNMS, albeit with application dependent loss function. This work mentioned faster RCNN based models but did not use them and hence the baseline is considerably lower than the current day works.  Hosang et al.~\cite{Hosang2017cvpr}, propose to absorb the entire NMS step into a neural network. The authors claim that the suppression width parameter can be better estimated by a neural net and hence it should be data dependent rather than an empirically chosen one. Even though this argument has merit, the adoption in state-of-the-art algorithms is still missing. Azadi et al.~\cite{Azadi_2017_CVPR} propose a similar method, where they use DPP as an alternative to NMS. However, in their method DPP is implemented as a trainable layer and not as a simple plug and play module.

Informative subset selection problems arise in many applications where a small number of items must be chosen to represent or cover a much larger set; for instance, text summarization~\cite{nenkova2006a,Lin10multidocumentsummarization}, document and image search~\cite{radlinski2008learning,Yue2008PredictingDS,KuleszaFixedDPPICML2011}, sensor placement~\cite{Guestrin2005}, viral marketing~\cite{KempeLNCS05}, and many others. Recently, probabilistic models extending determinantal point processes (DPPs)~\cite{macchi1975,TPPDaley} were proposed
for several such problems~\cite{KuleszaNIPS2010,KuleszaFixedDPPICML2011,GillenwaterEMNLP2012}. DPP was first used to characterize the Pauli exclusion principle, which states that two identical particles cannot occupy the same quantum state simultaneously~\cite{macchi1975}. DPPs offer computationally attractive properties, including exact and efficient computation of marginals~\cite{macchi1975}, sampling~\cite{hough2006,KuleszaFixedDPPICML2011}, and (partial) parameter estimation~\cite{KuleszaUAI2011LDP}. DPP has emerged as a powerful method for selecting a diverse subset from a ``ground set'' of items~\cite{KulTasDPP_1}.

\subsection{Determinantal Point Processes}
To define a determinantal point process (DPP) let us first consider the definition of a point process itself. A point process $\mathcal{P}$ on a ground set $\mathcal{Y}$ refers to a probability distribution on finite subsets of $\mathcal{Y}$. Let $\mathcal{Y}$ be a discrete set represented as $\mathcal{Y}=\{1,2,\dots,N\}$, then $\mathcal{P}$ defines a probability distribution on $2^\mathcal{Y}$, the powerset of $\mathcal{Y}$.

For $\mathcal{P}$ to be called a determinantal process, it should satisfy the following condition for all $A\subseteq\mathcal{Y}$:
\begin{equation}
    \mathcal{P}(A\subseteq\mathbf{Y})=\det(\textbf{K}_A)
\end{equation}
where, $\mathbf{Y}$ is a random subset drawn according to $\mathcal{P}$, $\textbf{K}$ is a real, symmetric $N\times N$ matrix indexed by the elements of $\mathcal{Y}$, and $\textbf{K}_A$ is the submatrix obtained from $\textbf{K}$ when only the entries indexed by elements of $A$ are considered. $\textbf{K}$ is referred to as the marginal kernel.

The above definition of DPP defines $\mathcal{P}$ in terms of marginal probabilities using $\textbf{K}$. There exists an alternative definition for a slightly restricted class of DPPs which allow us to model the probability of a subset directly. These are known as L-ensembles~\cite{BrunelMoitraStats2017} and are much easier to work with practically. We define $\mathcal{P}$ using L-ensembles as follows:
\begin{equation}
    \mathcal{P}_L(\mathbf{Y}=Y)\propto\det(\LL_Y)
\end{equation}
where, $\mathbf{Y}$ represents the random variable as earlier, $\LL$ is a real, symmetric $N\times N$ matrix indexed by elements of $\mathcal{Y}$, and $\LL_Y$ is similarly the submatrix of $\LL$ indexed by elements of $Y$. To satisfy the fact that probability measures must always be positive, $\LL$ has to be positive semidefinite (psd). The normalization constant for $\mathcal{P}$ can be obtained in closed form since
\begin{equation}
    \sum_{Y\subseteq\mathcal{Y}}\det(\LL_Y)=\det(\LL+\textbf{I})
\end{equation}
Thus, using L-ensembles we get a direct probability distribution on the subsets of $\mathcal{Y}$ as:
\begin{equation}
    \mathcal{P}_L(\mathbf{Y}=Y)=\frac{\det(\LL_Y)}{\det(\LL+\textbf{I})}
    \label{Eq:LEnsembles}
\end{equation}
Exact MAP inference of DPP is a NP-hard problem~\cite{EAM1995}. However, approximation of the DPP formulation, notably,
\begin{equation}\label{Eq:logDetDPP}
    f(S) = \log \det \LL_Y
\end{equation}
is a non-monotone submodular function~\cite{KulTasDPP_1}, which has been the function of choice for most of the work in this domain~\cite{GongNIPS14,feldman2018less}.

\section{Method}

We propose replacing GreedyNMS in detection pipelines with a DPP proposed in Eq.~\ref{Eq:logDetDPP}. Generally in a detection pipeline NMS is applied on final detections to filter them and keep only one detection per object. Faster RCNN, not only performs NMS on the final detections but also on the region proposals returned by the Region Proposal Network (RPN). We posit that the NMS after the RPN stage would gain with diversified selection, since its task is to retain all the informative regions. The second NMS which comes after the softmax stage just filters the boxes obtained for each class independently and hence does not gain with diversity preserving methods. Consequently, we replace the first stage NMS after the RPN layer in this work. As such it is here that we apply DPP. The basic idea is to use DPP to select or filter the proposals instead of NMS. Thus our ground set $\mathcal{Y}$ consists of the proposals returned by the RPN. GreedyNMS uses the box coordinates to compute an intersection over union metric and also the score provided by the RPN to filter the windows. We use the exact same two features for our method. To construct our $L$ matrix we make use of 2 features.
\begin{itemize}
    \item Scores for the proposals from the RPN ($s_i$)
    \item Intersection over union (IoU) of the proposals ($\mathrm{IoU}_{ij}$)
\end{itemize}
where $\{i \in \mathcal{Y}\}$. These features are then combined to form the $\LL$ matrix given by,
\begin{equation}\label{Eq:simEqMatr}
    \LL = \alpha [e^\mathbf{s} e^{\mathbf{s}^T}] \odot \mathbf{IoU}
\end{equation}
whose elements are written as follows:
\begin{equation}
    L_{ij} = \alpha e^{s_i}\mathrm{IoU}_{ij}e^{s_j}
    \label{Eq:simEq}
\end{equation}
where $\alpha > 1$ is a scaling constant provided to bias the selection process towards selecting larger subsets, and the values of $s_i\in (0,1)\forall i \in \mathcal{Y}$, $\mathbf{s}$ is a column vector with $s_i$ as its $i^{th}$ element, $e^{\mathbf{s}}$ represents the element-wise exponentiation of $\mathbf{s}$, $\mathbf{IoU}$ is a matrix composed of $\mathrm{IoU}_{ij}$, and $\odot$ represents the Hadamard product of matrices. Note that the interaction of the two score $s_i$ and $s_j$ can be combined in many different ways. In this work we use the exponent function to bring it closer to the smooth maximum approximation, along with the large weighting constant $\alpha$\footnote{https://en.wikipedia.org/wiki/Smooth\_maximum}.

\begin{Lemma}
    $\LL = \alpha [e^\mathbf{s} e^{\mathbf{s}^T}] \odot \mathbf{IoU}$
    is positive semidefinite.
\end{Lemma}
\begin{proof}
    The constituents of the $\LL$ matrix in the above manner can be proven to be individually positive semidefinite by the following three arguments. a) $e^\mathbf{s} e^{\mathbf{s}^T}$ is positive semidefinite since it is of the form $\mathbf{x}\mathbf{x}^T$, b) The $\mathbf{IoU}$ matrix, also known as the Jaccard similarity matrix, can be shown to be positive \textbf{definite} \cite{bouchard2013proof}, and c) According to the Schur product theorem\footnote{https://en.wikipedia.org/wiki/Schur\_product\_theorem}, the Hadamard product (elementwise multiplication product) of two positive semidefinite matrices is also positive semidefinite. Thus, the product $[e^\mathbf{s} e^{\mathbf{s}^T}] \odot \mathbf{IoU}$ is also positive semidefinite.

\end{proof}
The final probability of a selecting $Y\subseteq\mathcal{Y}$ can now be written as:
\begin{equation}
    \begin{split}
        \mathcal{P}(\mathbf{Y}=Y) &\propto \det(\alpha [e^{\mathbf{s}_Y} e^{\mathbf{s}_Y^T}]\odot \mathbf{IoU}_Y) = \alpha^{|Y|}\det( [e^{\mathbf{s}_Y} e^{\mathbf{s}_Y^T}]\odot\mathbf{IoU}_Y)
    \end{split}
    \label{Eq:FinalProbSel}
\end{equation}
Note that due to the determinant operation, the weighting term $\alpha$ is raised to the power $|Y|$, which is the size of the subset to be selected. Explicitly making the subset size influence the probability is important since the marginal gain decreases with increase in subset size. Hence, the weighting term acts as a counter to the diminishing marginal gain, which is due to the sub-modular nature of the objective function.

To obtain the set which maximizes the above probability we need to use some approximation technique. One choice is the simple greedy method. Before arriving at the final formulation we need the following lemmas.

\begin{Lemma}
    The principle sub-matrices of a psd matrix are also psd.
\end{Lemma}

According to this lemma any principle submatrix of $\LL$ indexed by the set $Y$ is also positive semidefinite. Hence, $\LL \succeq 0$ leads to all subsets $\LL_Y \succeq 0$.
\begin{Lemma}
    $\log \det \LL_Y$ for a psd matrix $\LL_Y$ is submodular.
    \label{Lemma3}
\end{Lemma}
\begin{proof}
    Submodularity of DPPs can be established by the geometrical argument as shown in~\cite{KulTasDPP_1}.
\end{proof}

Connecting all the lemmas, we can claim that all principal submatrices of $\mathbf{L} \succeq 0$ are themselves $\mathbf{L}_Y \succeq 0$. Finally, invoking Lemma.~\ref{Lemma3} and extending it to the current setting, we can maximize $\log\det \LL_Y$ to obtain the approximate MAP set. As such the final formulation for DPP based NMS is given by:
\begin{equation}\label{Eq.minFunc}
    \begin{split}
        \arg \max_Y~~ &\log\det \LL_Y = \log\det (\alpha [e^{\mathbf{s}_Y} e^{\mathbf{s}_Y^T}] \odot \mathbf{IoU}_Y) \\
    \end{split}
\end{equation}

We employ a greedy algorithm to maximize this cost function, where at every iteration we add the element which has the highest marginal gain with respect to the currently selected set. While greedy algorithms are not optimal in general, for monotone sub-modular problems they have well-defined approximation bounds~\cite{KulTasDPP_1}.
Our final algorithm is given as follows:

\begin{algorithm}
    \parbox{.48\linewidth}{
        \centering
        \begin{algorithmic}
            \STATE \textbf{Input:} RPN proposals $\mathcal{Y}$, RPN scores $\mathbf{s}$, parameter $\alpha$, maximum boxes $k$
            \STATE \textbf{Output:} Filtered proposals $Y$
            \STATE Compute $\mathbf{IoU}$ matrix using $\mathcal{Y}$
            \STATE $\LL\leftarrow\alpha[e^\mathbf{s} e^{\mathbf{s}^T}] \odot \mathbf{IoU}$
            \STATE $Y\leftarrow\mathrm{Greedy}(\mathcal{Y},k)$
            \STATE \textbf{return} $Y$
        \end{algorithmic}
    }
    \hfill
    \parbox{.48\linewidth}{
        \centering
        \begin{algorithmic}
            \STATE \textbf{Function} $\mathrm{Greedy}(\mathcal{Y},k)$\textbf{:}
            \STATE $X\leftarrow\mathcal{Y}, Y\leftarrow\emptyset$
            \WHILE{$|Y|<k$}
            \STATE $e\leftarrow\max_{i\in X}f(Y\cup i) - f(Y)$ ~~ Eq.\ref{Eq.minFunc}
            \IF{$f(Y\cup e) - f(Y)\leq0$}
            \STATE \textbf{return} $Y$
            \ENDIF
            \STATE $Y\leftarrow Y\cup e$
            \STATE $X\leftarrow X\setminus e$
            \ENDWHILE
            \STATE \textbf{return} $Y$
        \end{algorithmic}
    }
\end{algorithm}

We utilise a heap-based implementation to speed up the algorithm as proposed by~\citet{Minoux1978AcceleratedGA}. The additional check for positivity of the marginal gain in the greedy algorithm, ensures that the value of our currently selected set always increases at every iteration.


\section{Experiments and Results}
In this section we provide details about the experiments performed and discuss the various results obtained. We work with a standard PyTorch\footnote{https://pytorch.org/} version of faster-RCNN\footnote{https://github.com/jwyang/faster-rcnn.pytorch} and use VGG-16 as the backbone network. We maintain all the default settings to make the experiments as reproducible as possible. All our experiments are subsequently based on replacing the NMS module after the RPN block, with our own proposed method. We perform experiments on MS-COCO~\cite{MSCOCO} and PASCAL VOC 2007~\cite{PASCALVOC} datasets. 
In all cases we train the network for 6 epochs on the default training splits, which are mentioned in the respective dataset subsections. During training we do not use DPP. We replace the NMS module with DPP during test time. 
We believe that the merit of existing GreedyNMS is its simplicity and the fact that it does not need to be tuned much for any experiment. Consequently, we propose a similar setting where the default parameter configuration works well for most applications. We evaluate a few variants of our model to understand the different modes of its operation and then converge onto one model with default parameter recommendation.
The models in the experiments are as follows:
\begin{itemize}
    \item $\mathrm{gNMS}_x$: This is the standard Greedy NMS algorithm with a maximum of $x=\{300, 400\}$ selected windows. Note that $\mathrm{gNMS}_{300}$ is the default setting in most SOTA object detection pipelines with GreedyNMS.
    \item DPP$^\alpha_x$: This refers to DPP with bias factor $\alpha=5$, (Eq.\ref{Eq:simEq}), with a maximum of $x=\{300, 400\}$ selected boxes.
\end{itemize}
For all of the above models the number of input proposals (the ones returned by the RPN) are limited to a maximum of $|\mathcal{Y}|=6000$ windows. We present comparison against the previous works which are most similar to our in spirit. \textbf{Neural-NMS} represents the deep network based NMS proposed by Hosang et al.~\cite{Hosang2017cvpr}. They train their own deep network to replace Greedy NMS and plug it in after the detection step of Faster RCNN. This is a deviation from the generic way of using NMS, where it is plugged after the RPN but before the detection stage. 
\textbf{MP-NMS} refers to the message passing based NMS algorithm proposed by Rothe et al.~\cite{Rothe2014NonmaximumSF}. We also compare against the end to end integration of convolution network, deformable parts model and NMS into one unified pipeline, proposed by Wan et al.~\cite{WanCVPR2015}. Though this method, denoted as \textbf{CN-DPM-NMS}, does not use F-RCNN like network, but the results can still work as a baseline comparison. Finally, \textbf{LDDP} refers to the pipeline proposed by Azadi et al.~\cite{Azadi_2017_CVPR} where they use a trainable DPP layer as an alternative to NMS.
All experiments were performed on a system with a i7-6850k CPU, a GTX 1080 Ti GPU and 64GB RAM. We implement DPP in C++ using the Eigen3 framework and run it on the CPU. When compared to a basic C++ CPU implementation of NMS we get comparable runtime upto approximately 100 selections for which NMS takes about 0.3s/image whereas DPP takes about 0.5s/image. The runtime of DPP however scales significantly with the number of selected proposals since the complexity involved is approximately $O(k^4)$, where $k$ is the number of proposals selected.


\subsection{MS-COCO}
For the MS-COCO dataset the model was trained on the training and valminusminival data splits and was tested on the minival split. 
In the results AP$_{0.5}$ represents average precision (AP) calculated considering 50\% overlap with ground truth. AP$_{0.5}^{0.95}$ represents AP averaged over multiple overlap thresholds ranging from 50\% to 95\% in steps of 5\%. The results for multi-class classification are shown in Table.~\ref{tab:nms_vs_dpp_ms_coco}. Results for MS-COCO person detection class has been reported by several authors and hence we also report it separately in Table.~\ref{tab:nms_vs_dpp_ms_coco_person}.

\begin{table}
    \parbox{.48\linewidth} {
        \centering
        \begin{tabular}{|l|c|c|}
            \hline
            Model                                 & AP$_{0.5}$ & AP$_{0.5}^{0.95}$ \\
            \hline\hline
            \bf{gNMS}$_{300}$                     & 47.7       & 27.3              \\
            \bf{gNMS}$_{400}$                     & 48.0       & 27.4              \\
            \bf{DPP$_{300}^{5}$}                  & 47.8       & 27.4              \\
            \bf{DPP$_{400}^{5}$}                  & \bf{48.1}  & \bf{27.5}         \\
            \bf{Neural-NMS~\cite{Hosang2017cvpr}} & -          & 24.3              \\

            \bf{LDDP~\cite{Azadi_2017_CVPR}}      & 32.2       & 15.5              \\
            \hline
        \end{tabular}
        \caption{NMS vs DPP experiments on MS COCO (All Classes)}
        \label{tab:nms_vs_dpp_ms_coco}
    }
    \hfill
    \parbox{.48\linewidth}{
        \centering
        \begin{tabular}{|l|c|c|}
            \hline
            Model                                 & AP$_{0.5}$ & AP$_{0.5}^{0.95}$ \\
            \hline\hline
            \bf{gNMS}$_{300}$                     & 69.8       & 40.2              \\
            \bf{gNMS}$_{400}$                     & 70.0       & 40.3              \\
            \bf{DPP$_{300}^{5}$}                  & 69.9       & 40.3              \\
            \bf{DPP$_{400}^{5}$}                  & \bf{70.2}  & \bf{40.6}         \\
            \bf{Neural-NMS~\cite{Hosang2017cvpr}} & 67.3       & 36.9              \\
            \hline
        \end{tabular}
        \caption{NMS vs DPP experiments on MS COCO (Persons)}
        \label{tab:nms_vs_dpp_ms_coco_person}
    }
\end{table}

\subsection{PASCAL VOC}
For PASCAL VOC 2007 we perform several experiments. We start off by evaluating Greedy NMS vs several variants of DPP over each class individually. 
For these experiments Faster-RCNN was trained on the training and validation sets and tested on the test set for PASCAL VOC 2007. For assigning proposed bounding boxes to ground truth detections PASCAL VOC considers overlaps greater than 50\% to be correct detections. This evaluation criteria is denoted as AP$_{0.5}$. Table.~\ref{tab:nms_vs_dpp_pascal_voc_full} shows the results of class wise performance. 
Average performance across all classes along with comparative methods are shown in Table.~\ref{tab:nms_vs_dpp_pascal_voc_full_avg}.

\begin{table}
    \begin{center}
        \begin{scriptsize}
            \begin{tabular}{|l|c|c|c|c|c|c|c|c|c|c|}
                \hline
                \scriptsize{Model}   & \tiny{aeroplane}   & \tiny{bicycle} & \tiny{bird}    & \tiny{boat}      & \tiny{bottle}  & \tiny{bus}         & \tiny{car}     & \tiny{cat}     & \tiny{chair}   & \tiny{cow}       \\
                \hline\hline
                \bf{gNMS}$_{300}$    & 67.66              & 77.39          & 67.15          & 54.36            & 54.43          & \textbf{78.40}     & \textbf{85.52} & {85.67}        & {48.45}        & \textbf{79.78}   \\
                \bf{gNMS}$_{400}$    & 68.18              & 77.96          & \textbf{67.89} & 54.61            & 54.72          & 78.17              & 85.50          & \textbf{85.98} & \textbf{48.61} & 79.73            \\
                \bf{DPP$_{300}^{5}$} & 69.93              & 77.22          & 65.75          & 54.79            & 55.43          & 78.25              & 85.05          & 82.40          & 47.93          & 80.36            \\
                \bf{DPP$_{400}^{5}$} & \textbf{69.94}     & \textbf{78.51} & 65.42          & \textbf{55.13}   & \textbf{55.49} & 77.90              & 85.29          & 83.53          & 48.01          & 78.20            \\
                \hline\hline
                \scriptsize{Model}   & \tiny{diningtable} & \tiny{dog}     & \tiny{horse}   & \tiny{motorbike} & \tiny{person}  & \tiny{pottedplant} & \tiny{sheep}   & \tiny{sofa}    & \tiny{train}   & \tiny{tvmonitor} \\
                \hline\hline
                \bf{gNMS}$_{300}$    & 61.50              & 78.89          & 82.14          & 75.61            & 77.26          & 40.65              & 70.42          & \textbf{63.77} & 74.94          & 72.19            \\
                \bf{gNMS}$_{400}$    & 61.10              & 78.59          & 82.32          & 75.39            & 77.23          & 40.96              & 70.16          & \textbf{63.77} & \textbf{75.28} & 71.76            \\
                \bf{DPP$_{300}^{5}$} & 63.35              & 80.97          & \textbf{83.10} & 75.67            & \textbf{77.60} & 42.22              & \textbf{71.45} & 63.71          & 74.89          & 72.61            \\
                \bf{DPP$_{400}^{5}$} & \textbf{63.91}     & \textbf{81.26} & 83.06          & \textbf{76.17}   & 77.54          & \textbf{42.63}     & 70.34          & 63.49          & 75.11          & \textbf{72.49}   \\
                \hline
            \end{tabular}
        \end{scriptsize}
    \end{center}
    \caption{NMS vs DPP experiments on PASCAL VOC 2007 (Classwise)}
    \label{tab:nms_vs_dpp_pascal_voc_full}
\end{table}
\begin{table}
    \begin{center}
        \begin{tabular}{|l|c|}
            \hline
            Model                              & AP$_{0.5}$     \\
            \hline\hline
            \bf{gNMS}$_{300}$                  & 69.81          \\
            \bf{gNMS}$_{400}$                  & 69.90          \\
            \bf{DPP$_{300}^{5}$}               & 70.13          \\
            \bf{DPP$_{400}^{5}$}               & \textbf{70.17} \\
            \bf{MP-NMS~\cite{RasmusMPNMS}}     & 56.14          \\
            \bf{CN-DPM-NMS~\cite{WanCVPR2015}} & 46.50          \\
            \bf{LDDP~\cite{Azadi_2017_CVPR}}   & 62.21          \\
            \hline
        \end{tabular}
    \end{center}
    \caption{Average performance on PASCAL VOC 2007}
    \label{tab:nms_vs_dpp_pascal_voc_full_avg}
\end{table}
\begin{figure}[ht!]
    \parbox{.48\linewidth}{
        \centering
        \includegraphics[width=\linewidth]{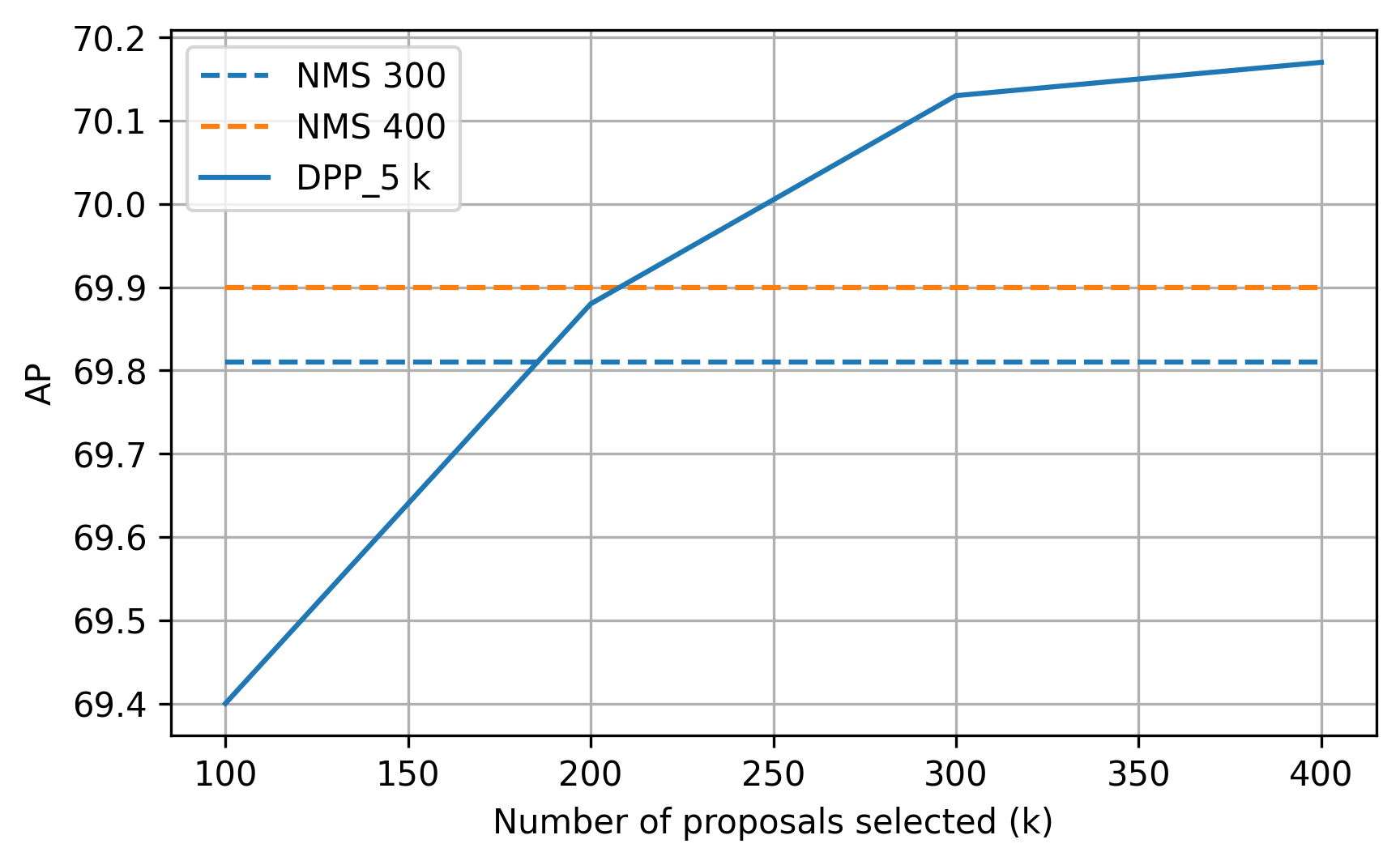}
        \caption{Comparison of varying the maximum window parameter $k$ in our algorithm.}
        \label{Fig:MaxWindow}
    }
    \hfill
    \parbox{.48\linewidth}{
        \centering
        \vspace{12pt}
        \includegraphics[width=\linewidth]{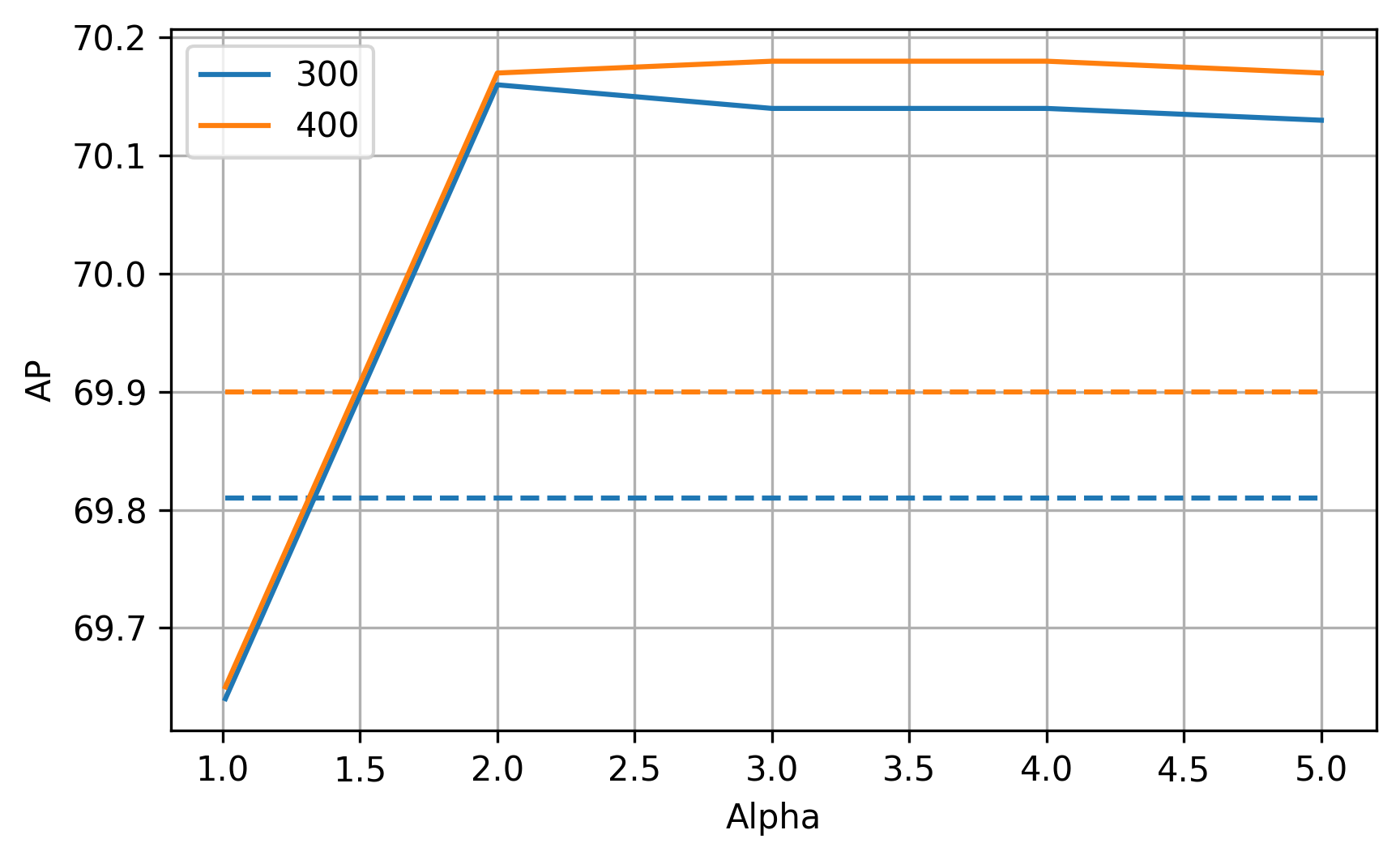}
        \caption{Comparison of varying the scaling parameter $\alpha$ in our algorithm. The horizontal dotted lines denote GreedyNMS.}
        \label{Fig:AlphaVar}
    }
\end{figure}
\subsection{Varying the maximum window and scaling parameters}
We perform more experiments to identify the core strengths of the proposed algorithm. The maximum number of windows returned by the algorithm is a parameter, which has a direct implication on the run-time of the algorithm. As such, the minimum value at which acceptable results are obtained needs to be selected. Keeping $\alpha = 5$, we run the algorithm with different values of $k\in\{100,200,300,400\}$. The results are shown in Fig.~\ref{Fig:MaxWindow}. Note that, for the setting $k=200$, our algorithm already beats gNMS$_{300}$ and is almost at par with gNMS$_{400}$. This is the key contribution of introducing diversified window selection in the NMS algorithm, wherein, a diverse set of lesser number of proposal windows ($k=200$) outperform a larger set of proposal windows ($k=300$) selected by GreedyNMS.

Similarly, we also perform experiments to observe the effect of the scaling parameter $\alpha$ on the detection performance. We test different values of $\alpha \in \{1.01, 2, 3, 4, 5\}$ while keeping the maximum number of windows $k$ fixed. The results are shown in Fig.~\ref{Fig:AlphaVar}. The proposed method beats GreedyNMS for $\alpha > 1.5$ for both 300 and 400 region proposal selections. 

\begin{figure}
    \centering
    \includegraphics[width=0.5\linewidth]{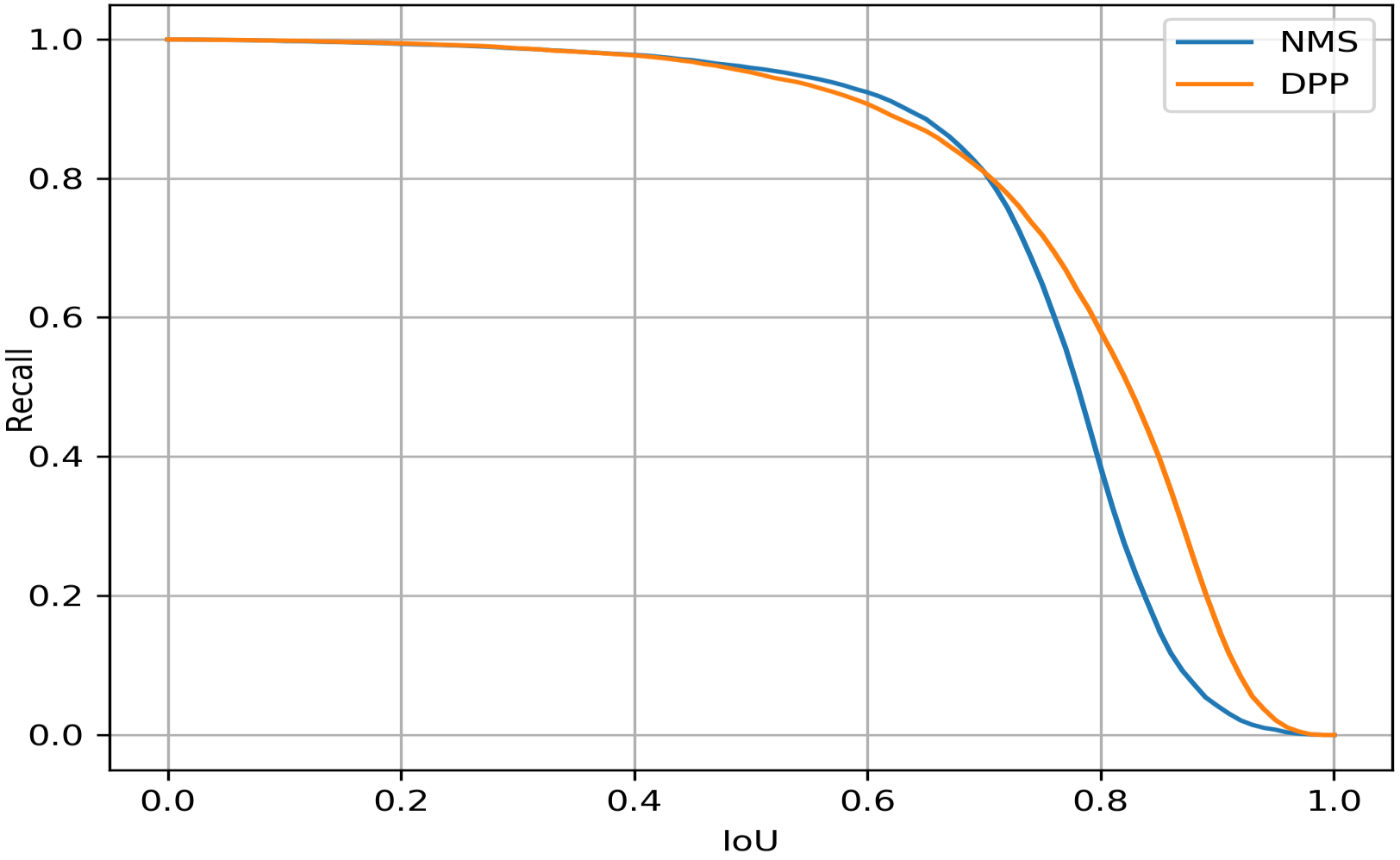}
    \caption{IoU vs Recall plot for gNMS$_{400}$ and DPP$_{400}^5$}
    \label{Fig:IoURecall}
\end{figure}

\subsection{IoU vs Recall}
In ~\cite{RasmusMPNMS} the authors propose evaluating the recall rate of detections at different IoU thresholds to measure how well fitting the selected bounding boxes are. We perform a similar evaluation, where we plot the recall with respect to the ground truth boxes against varying IoU thresholds (Fig.~\ref{Fig:IoURecall}). As NMS/DPP is applied on the RPN proposals in Faster-RCNN, we directly consider these proposals before any bounding box regression for this experiment. The IoU threshold determines whether a predicted bounding box is matched to a ground truth object or not. The AUC scores for the two curves are 0.7575 for Greedy NMS and 0.7869 for the DPP based method. 
In addition to having higher AUC we also note that the DPP based method becomes especially better when more precise bounding boxes are required ($IoU>0.7$). This indicates that DPP chooses better fitting bounding boxes than Greedy NMS.


\section{Qualitative Results}
We show a few qualitative results in Fig.~\ref{Fig:CocoComp1} using similar parameter settings as used for all the previous results. We select images from the MS-COCO validation set and plot the region boundaries found by the two competing methods, namely gNMS$_{400}$ and DPP$_{400}^{5}$. It is interesting to observe that DPP based selection works well when there is large overlap between two correct detections. DPP was able to remove some extraneous windows, such as the extra person detection for the tennis player blue cluster in Fig.~\ref{Fig:CocoComp1}. Similarly, it selects only meaningful windows for the collection of people in the bottom right image in the blue cluster. For images with very simple / few detections, both the methods perform at par. A few examples where NMS still performs better are shown in the green cluster in Fig.~\ref{Fig:CocoComp1}.

\begin{figure}[hbpt!]
    \centering
    \includegraphics[width=0.9\columnwidth]{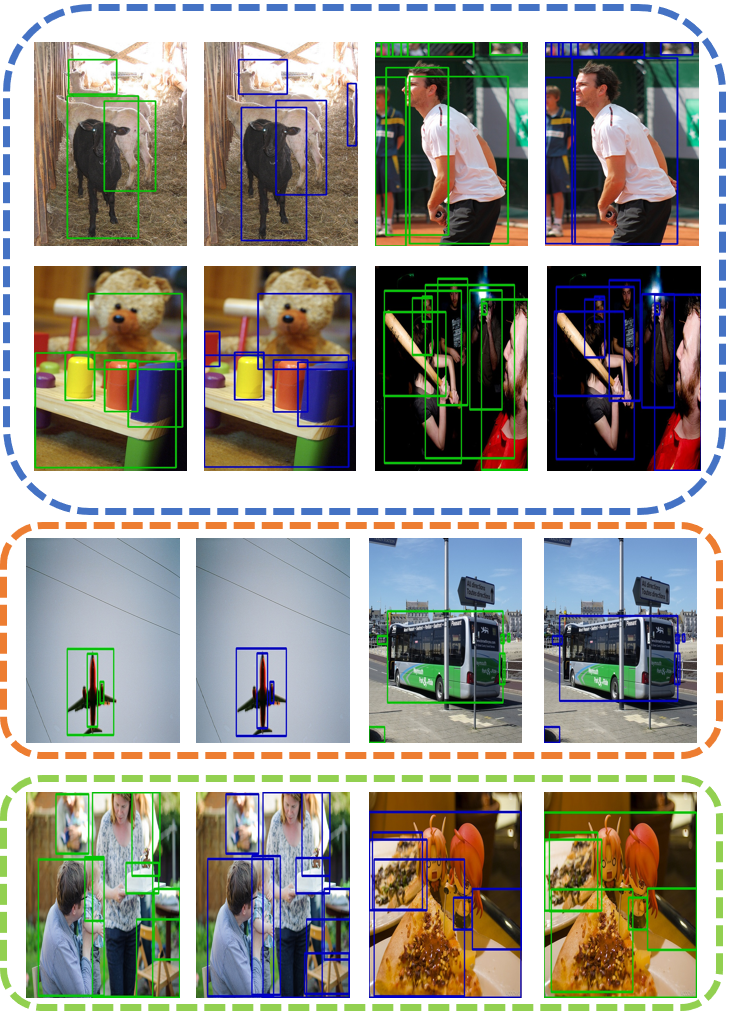}
    \caption{Qualitative results (best viewed in color). Blue boxes are produced by our method DPP$_{400}^{5}$, green boxes are produced by gNMS$_{400}$. Blue dotted cluster represents results where DPP$_{400}^{5}$ performs better than gNMS$_{400}$. Brown cluster represents similar performance. Green cluster represents cases where gNMS$_{400}$ seems to perform better, although the person detection is still superior for DPP$_{400}^{5}$.}
    \label{Fig:CocoComp1}
\end{figure}

\section{Conclusion and Future Work}
We propose a novel integration of DPP based diverse set selection technique into the NMS paradigm. We formulate a principled cost function which uses the same two features which the traditional NMS routines use, and show that this formulation can be driven to improve on NMS accuracy by carefully selecting the bias parameter $\alpha$ which promotes larger subsets. The comparative results against Greedy NMS as well as other recent methods prove that the proposed method is working at par or superior than most other methods.

\bibliography{egbib}
\end{document}